\documentclass[11pt,letterpaper]{article}

\usepackage[margin=1in]{geometry}
\usepackage[T1]{fontenc}
\usepackage{times}
\usepackage{microtype}
\usepackage{amsmath,amssymb,amsthm}
\usepackage{graphicx}
\usepackage{xcolor}
\usepackage{booktabs}
\usepackage{tablefootnote}
\usepackage{placeins}
\usepackage{algorithm}
\usepackage{algorithmic}
\usepackage[authoryear,round]{natbib}
\usepackage[hidelinks]{hyperref}

\newtheorem{theorem}{Theorem}[section]
\newtheorem{lemma}[theorem]{Lemma}
\newtheorem{proposition}[theorem]{Proposition}
\newtheorem{corollary}[theorem]{Corollary}
\newtheorem{conjecture}[theorem]{Conjecture}
\theoremstyle{definition}
\newtheorem{definition}[theorem]{Definition}
\newtheorem{example}{Example}
\theoremstyle{remark}
\newtheorem{remark}[theorem]{Remark}
\theoremstyle{plain}

\usepackage{amsmath,amsfonts,bm}

\def\eqref#1{equation~\ref{#1}}

\def\1{\bm{1}}

\DeclareMathAlphabet{\mathsfit}{\encodingdefault}{\sfdefault}{m}{sl}
\SetMathAlphabet{\mathsfit}{bold}{\encodingdefault}{\sfdefault}{bx}{n}

\def\K{\mathcal{K}}

\newcommand{\E}{\mathbb{E}}

\newcommand{\R}{\mathbb{R}}

\newcommand{\Reg}{\mathsf{Reg}}

\makeatletter
\@ifundefined{theorem}{\newtheorem{theorem}{Theorem}[section]}{%
  \counterwithin{theorem}{section}%
}

\@ifundefined{claim}{}{}
\@ifundefined{subclaim}{}{}
\@ifundefined{proposition}{}{}
\@ifundefined{lemma}{\newtheorem{lemma}[theorem]{Lemma}}{}
\@ifundefined{corollary}{\newtheorem{corollary}[theorem]{Corollary}}{}
\@ifundefined{conjecture}{\newtheorem{conjecture}[theorem]{Conjecture}}{}
\@ifundefined{observation}{}{}
\@ifundefined{fact}{}{}

\@ifundefined{definition}{\newtheorem{definition}[theorem]{Definition}}{}
\@ifundefined{construction}{}{}
\@ifundefined{example}{}{}
\@ifundefined{algorithm1}{}{}
\@ifundefined{protocol}{}{}
\@ifundefined{assumption}{}{}
\@ifundefined{remark}{}{}

\@ifundefined{proof}{%
  \newenvironment{proof}[1][Proof]
  {\par\noindent\textit{#1.}\ }
  {\hfill$\square$\par}%
}{}
\makeatother

\newcommand{\norm}[1]{\left\lVert#1\right\rVert}

\makeatletter
\newcommand{\floatconts}[3]{%
  \def\arxiv@figure{figure}%
  \ifx\@captype\arxiv@figure
    {\centering #3\par}\medskip
    #2\label{#1}%
  \else
    #2\label{#1}%
    #3%
  \fi
}
\makeatother

\hypersetup{
  pdftitle={Tracking a Small Set of Experts in Online Convex Optimization},
  pdfauthor={Angelos Assos, Anand Brahmbhatt, Elad Hazan},
  pdfsubject={Online convex optimization and tracking regret},
  pdfkeywords={online convex optimization, tracking regret, exp-concavity,
    computational lower bounds, planted clique}
}

\title{Tracking a Small Set of Experts in Online Convex Optimization}
\author{%
  \begin{tabular}{ccc}
    Angelos Assos\textsuperscript{1} &
    Anand Brahmbhatt\textsuperscript{1} &
    Elad Hazan\textsuperscript{1,2} \\
    \small\href{mailto:aa2279@princeton.edu}{\texttt{aa2279@princeton.edu}} &
    \small\href{mailto:ab7728@princeton.edu}{\texttt{ab7728@princeton.edu}} &
    \small\href{mailto:ehazan@princeton.edu}{\texttt{ehazan@princeton.edu}} \\[6pt]
    \multicolumn{3}{c}{\small\textsuperscript{1}Princeton University} \\
    \multicolumn{3}{c}{\small\textsuperscript{2}Google DeepMind Princeton}
  \end{tabular}%
}
\date{}

\begin{document}
\maketitle
\begin{abstract}
For the fundamental problem of prediction from expert advice, Freund asked whether one could compete with a comparator that switches $k$ times among only $m$ experts. For a discrete set of experts, this problem was resolved by \cite{bousquet2002tracking}, who gave an efficient algorithm with optimal regret matching the information-theoretic limit. 

We study the analogous problem in online convex optimization, where the action space is continuous. An inefficient algorithm can be derived by discretizing the action space and applying the Bousquet--Warmuth algorithm, giving regret of $\widetilde{\mathcal{O}}(md+k)$, while the best known efficient algorithm for this problem has regret $\widetilde{\mathcal{O}}(dk)$.  We show that this computational inefficiency is inherent, establishing a separation between the information-theoretic and polynomial-time regret achievable for this problem. Specifically, under the hypergraphic planted-clique conjecture, we show that for every polynomial-time algorithm there exist instances for which it incurs at least $(k\sqrt d)^{1-o(1)}$ regret.
\end{abstract}


\section{Introduction}

On each round $t$ of an online learning problem, a learner $\mathcal A$ chooses a decision $x_t$ from a set $\K$, observes a loss function $f_t$, and incurs $f_t(x_t)$. Its regret against a comparator path $u_{1:T}\in \K^T$ is
\[
\Reg_T(\mathcal A;u_{1:T})
:=
\sum_{t=1}^T f_t(x_t)
-
\sum_{t=1}^T f_t(u_t).
\]
External regret compares the learner with the best fixed decision in hindsight, restricting the comparator to a constant path. This benchmark is well understood, with upper and lower bounds characterizing the achievable rates in standard online learning settings~\citep{hazan2023introductiononlineconvexoptimization}. In a changing environment, however, a fixed decision may be too weak a benchmark. Dynamic regret allows the comparator to change over time, but sublinear regret is impossible in general when these changes are unrestricted.

A natural intermediate benchmark, introduced by \citet{bousquet2002tracking} for prediction from a finite set of \(N\) experts, is to compete with a comparator that switches at most \(k\) times among at most \(m\) distinct experts. When \(m\ll k\), the comparator repeatedly returns to the same small set, raising the question of whether the learner can exploit this recurrence. Under log loss, the worst-case information-theoretic regret lower bound is of order
\[
\underbrace{m\log(N/m)}_{\text{identify the pool}}
+
\underbrace{k\log m}_{\text{choose within the pool}}
+
\underbrace{k\log(T/k)}_{\text{locate the switches}}
\]
in the regime \(2\le m\ll N\) and \(m\ll k\ll T\).
There are $\binom{N}{m}$ choices for the pool of experts and
$\binom{T-1}{k}$ choices for the switch locations among the $T-1$
gaps between consecutive rounds; within the pool, there are $m$
choices for the initial expert and $m-1$ choices at each switch.
Multiplying gives $\binom{N}{m}m(m-1)^k\binom{T-1}{k}$ possible
descriptions of a comparator, whose logarithm gives the order of
the information-theoretic lower bound above
\citep[Section~4.4]{bousquet2002tracking}.
Surprisingly, Bousquet and Warmuth show that an efficient algorithm
can attain a matching upper bound for mixable losses, up to
loss-dependent constants and additional logarithmic terms.

We study the analogous problem in online convex optimization (OCO), where $\K\subseteq\mathbb R^d$ is a continuous convex decision set and the losses $f_t$ are convex. For constant $m$ and Lipschitz exp-concave losses on a bounded domain, discretizing the domain and applying the Bousquet--Warmuth algorithm gives regret $\widetilde O(d+k)$. We make this reduction precise in Section~\ref{sec:exp-concave-motivation}. However, the discretization can contain exponentially many points in $d$,
making this implementation computationally inefficient. On the other hand, polynomial-time algorithms achieve \(\widetilde O(dk)\) regret for \(k\ge1\) (Theorem~\ref{thm:ons-switching-upper}). These bounds leave open the best achievable dependence on \(d\) and \(k\) under polynomial-time computation. We therefore ask:
\begin{center}
\itshape
What is the best regret rate a computationally efficient algorithm can achieve in this setting?
\end{center}

\subsection{Overview of Our Results}
\label{sec:results-overview}

Our main contribution is a conditional separation between information-theoretic regret and the regret achievable by polynomial-time algorithms, even when the comparator uses only two distinct points.

\paragraph{Computational lower bound.}
Assuming hypergraphic planted-clique hardness for every fixed
order $r\ge4$, Theorem~\ref{thm:r-hpc-lower-bound} rules out
polynomial-time $\widetilde O(k^a d^b)$ expected regret along
$k=\Theta(d^\theta)$ whenever $\theta a+b<\theta+1/2$,
for fixed rational $\theta>0$ and fixed $a,b\ge0$.
In particular, no such guarantee holds for all $(k,d)$
if $a<1$ or $b<1/2$, by taking $\theta$ sufficiently
large or small, respectively. This holds even for $m=2$ and bounded
$1$-exp-concave rank-one quadratics on the Euclidean unit ball,
with $T=\operatorname{poly}(d)$.
Figure~\ref{fig:regret-exponents} illustrates the computational
separation for $\theta>1/2$.

\paragraph{Efficient upper bounds.}
Standard adaptive-regret methods give $\widetilde O(k)$ regret
for strongly convex losses (Theorem~\ref{prop:strong}) and
$\widetilde O(dk)$ regret for exp-concave losses
(Theorem~\ref{thm:ons-switching-upper}).
A gap remains between our computational lower bound and the polynomial-time upper bounds in the exp-concave case, leaving the optimal polynomial-time rate open.

\vspace{2mm}
Table~\ref{tab:regret-comparison} summarizes the results across the three curvature regimes. For convex and strongly convex losses, efficient algorithms already attain the information-theoretically optimal dependence on \(k\), with matching lower bounds. For exp-concave losses, by contrast, recurrence creates a gap between the information-theoretic and efficiently achievable rates.

\begin{table}[htb]
\centering
\caption{Regret bounds for constant $m\ge2$ and $k\ge1$,
with loss and domain parameters fixed.}
\label{tab:regret-comparison}
\small
\setlength{\tabcolsep}{4pt}
\begin{tabular}{@{}lcccc@{}}
\toprule
Loss class
& \shortstack{Information-theoretic\\lower bound}
& \shortstack{Polynomial-time\\upper bound}
& \shortstack{Computational\\lower bound} \\
\midrule
Convex
& \shortstack{$\Omega(\sqrt{kT})$[\citenum{pasteris2024switching}]\tablefootnote{Partition the horizon into $k+1$ equal blocks and use independent random linear losses on each block. Each block contributes $\Omega(\sqrt{T/k})$ regret, with its best comparator in $\{-1,1\}$, yielding $\Omega(\sqrt{kT})$.}}
& \shortstack{$O(\sqrt{kT})$ [\citenum{pasteris2024switching}]}
& -- \\
\addlinespace
Strongly convex
& $\Omega(k)$\tablefootnote{On block $j$, use $f_t(x)=(x-\sigma_j)^2$ with independent uniform $\sigma_j\in\{-1,1\}$. The first round of each block incurs constant expected regret, yielding $\Omega(k)$.}
& \shortstack{$\widetilde O(k)$ [Thm.~\ref{prop:strong}]}
& -- \\
\addlinespace
Exp-concave
& $\widetilde\Omega(d+k)$ [\citenum{bousquet2002tracking}]
& \shortstack{$\widetilde O(dk)$ [Thm.~\ref{thm:ons-switching-upper}]}
& \shortstack{$(k\sqrt{d})^{1-o(1)}$
[Thm.~\ref{thm:r-hpc-lower-bound}]} \\
\bottomrule
\end{tabular}
\end{table}
\FloatBarrier

\begin{figure}[!htb]
\floatconts
{fig:regret-exponents}
{\caption{Regret exponents for $k=\Theta(d^\theta)$ (fixed rational
$\theta>0$), ignoring logarithmic factors.
Theorem~\ref{thm:r-hpc-lower-bound} conditionally rules out every
fixed exponent below the dashed line. Shading marks the
computational separation for $\theta>1/2$.}}
{\includegraphics[width=0.8\linewidth]{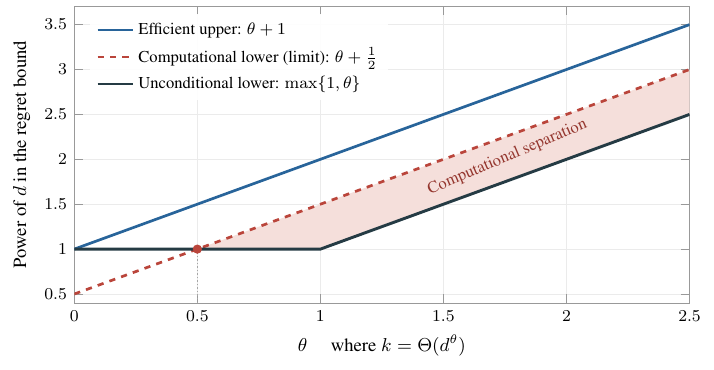}}
\end{figure}

\subsection{From Mixable Experts to Exp-Concave OCO}
\label{sec:exp-concave-motivation}

The Bousquet--Warmuth algorithm for a finite set of experts
relies on \emph{mixability} of losses. For an $\alpha$-mixable
prediction loss, expert predictions can be combined using
probability weights $p_i$ into a single prediction whose loss,
for every outcome, is at most
$-\alpha^{-1}\log\sum_i p_i e^{-\alpha\ell_i}$.
Exp-concavity supplies precisely this property in the OCO setting, i.e. for any
$v_1,\ldots,v_N\in\K$ and probability weights $p_1,\ldots,p_N$,
concavity of $e^{-\alpha f}$ gives
\[
    f\!\left(\sum_{i=1}^N p_i v_i\right)
    \le -\frac{1}{\alpha}
        \log\left(\sum_{i=1}^N p_i e^{-\alpha f(v_i)}\right).
\]

To tackle the OCO problem, treat each point in a finite net of $\K$ as an expert.
Assuming exp-concavity of the loss function, the inequality above allows the learner to run Bousquet--Warmuth
by playing weighted averages of these points.
For Lipschitz losses on a bounded domain, a sufficiently fine
net has $\log N=O(d\log T)$ and approximates every comparator
with at most constant additional loss, without increasing
its number of distinct points or switches.\footnote{For
$G$-Lipschitz losses on a domain of diameter $D$, choose
$\varepsilon=1/(GT)$. There exists an $\varepsilon$-net
contained in $\K$ with $N\le(1+2D/\varepsilon)^d
=(1+2DGT)^d$ points. Mapping each distinct comparator point
to a fixed net point within distance $\varepsilon$ adds at
most $GT\varepsilon=1$ to its cumulative loss and cannot
increase the number of distinct points or switches.}
The Bousquet--Warmuth guarantee then gives
$\widetilde O(m\log N+k)=\widetilde O(md+k)$ regret,
with loss and domain parameters fixed.
However, the net can be exponentially large in $d$,
making this implementation inefficient.

\subsection{Overview of the Hardness Reduction}
\label{sec:techniques}

We reduce hypergraphic planted-clique detection to online convex
optimization. The hypergraphic planted-clique
conjecture~\citep{luo2020hpc} asserts that no polynomial-time
algorithm can distinguish a random hypergraph on $N$ vertices
from one with a planted clique whose size is polynomially
below $\sqrt N$. We map the hypergraph to bounded exp-concave
losses on the Euclidean unit ball.

We measure the learner's cumulative gain over always playing
the origin by $G=\sum_{t=1}^T(f_t(0)-f_t(x_t))$.
Under the planted model, the clique determines a reference
sequence $u_{1:T}$ that plays one fixed point on selected
blocks and the origin elsewhere. This comparator uses only
two distinct decisions and, with high probability, has a
large advantage $\Delta=\sum_{t=1}^T(f_t(0)-f_t(u_t))$.
The identity
\[
    G=\Delta-\Reg_T(u_{1:T})
\]
then implies that sufficiently small expected regret forces
a large expected gain. Under the null model, however, the
labels are independent unbiased signs, making a large
positive gain unlikely. Thus, sufficiently small regret 
would let us detect the planted clique through the learner’s observable gain.
Section~\ref{sec:hardness} gives the construction and formal
argument.

\section{Related Work}\label{sec:related_work}

\paragraph{Tracking experts and long-term memory.}
\cite{herbster1998tracking} introduced Fixed Share for comparator sequences with a bounded number of switches. For comparators that repeatedly return to a small pool of experts, a problem posed by Yoav Freund, \cite{bousquet2002tracking} introduced mixing past posteriors to reuse information about previously successful experts. Extensions cover stateful experts, Bayesian formulations, mirror descent, and growing expert sets~\citep{koolen2010freezing,koolen2012putting,cesabianchi2012mirror,mourtada2017efficient}. \cite{zheng2019longterm} resolved the question of \cite{warmuth2014shifting} on exploiting favorable stochastic structure and extended the framework to bandit feedback. \cite{robinson2021memory} reduced the per-round cost of mixing-past-posteriors methods to linear, while \cite{herbster2020multitask} extended long-term memory to multitask learning and structured infinite RKHS classes.

\paragraph{Shifting and adaptive regret in OCO.}
A parallel literature studies changing comparators in continuous decision spaces without exploiting recurrence. \cite{herbster2001linear} studied shifting regret for online linear prediction, and \cite{gyorgy2016shifting} developed an OCO framework measuring comparator complexity by deviation from prescribed dynamics. Adaptive regret requires low static regret on every time interval~\citep{hazan2009changing}; \cite{daniely2015strongly} strengthened these guarantees to depend on interval length. Summing over a piecewise-constant comparator's segments yields shifting-regret bounds, but pays the static learning cost on each segment. \cite{pasteris2024switching} obtained optimal shifting regret for general convex OCO, while \cite{yang2026logarithmic} obtained logarithmic shifting regret for exp-concave and strongly convex OCO.

\paragraph{Static OCO and curvature.}
Online gradient descent achieves $O(\sqrt{T})$ static regret for general convex losses~\citep{zinkevich2003online}. Strongly convex losses admit dimension-free $O(\log T)$ regret, while Online Newton Step achieves $O(d\log T)$ for exp-concave losses~\citep{hazan2007logarithmic}. Consequently, relearning on every segment is inexpensive for strongly convex losses but incurs a dimension-dependent cost in the general exp-concave case. See \cite{hazan2023introductiononlineconvexoptimization} for a treatment of static, adaptive, and dynamic regret in OCO.

\paragraph{Planted clique and average-case hardness.}
Planted clique is a classical source of average-case computational hardness. \citet{alon1998hidden} gave a polynomial-time spectral recovery algorithm above the $\sqrt N$ scale; \citet{jerrum1992large} provided hardness evidence by analyzing the failure of the Metropolis process. \citet{Hazan2011HowHard} used planted-clique hardness to show that finding an approximate Nash equilibrium with near-optimal social welfare is hard. Hypergraphic planted clique has been used to establish statistical--computational gaps in tensor estimation and clustering~\citep{zhang2018tensor,luo2022tensor}. \citet{luo2020hpc} formulated its average-case hardness conjecture and studied its relation to ordinary planted clique, while \citet{brennan2020reducibility} developed related reductions based on planted-clique variants.

\section{Preliminaries}
\label{sec:benchmark}

\paragraph{Notation.}
We write $[n]=\{1,\ldots,n\}$ and use $\1\{A\}$ for the indicator of an event $A$.
The Euclidean inner product and norm are denoted by $\langle x,y\rangle$ and $\norm{x}_2$,
and $B_2^d=\{x\in\R^d:\norm{x}_2\le1\}$ is the Euclidean unit ball.
The notation $\widetilde O$ and $\widetilde\Omega$ suppresses polylogarithmic factors.

\paragraph{Online learning protocol.}
We consider online convex optimization over a nonempty compact convex set $\K\subseteq\R^d$ of diameter $D$. On round $t$, the learner chooses $x_t\in \K$ based on previous rounds, after which a differentiable convex loss $f_t$ is revealed and the learner incurs $f_t(x_t)$. The aim is to keep the cumulative loss close
to that of a comparator whose decision may change over time.

We study this problem for exp-concave losses and, as a special case,
strongly convex losses with bounded gradients. We state the definitions below.

\begin{definition}[Exp-concavity]
\label{def:exp-concavity}
For $\alpha>0$, a function $f:\K\to\R$ is \emph{$\alpha$-exp-concave} if
$x\mapsto e^{-\alpha f(x)}$ is concave on $\K$.
\end{definition}

\begin{definition}[Strong convexity]
\label{def:strong-convexity}
For $\lambda>0$, a differentiable function $f:\K\to\R$
is \emph{$\lambda$-strongly convex} if, for every $x,y\in \K$,
\[
f(y)\ge
f(x)+\langle\nabla f(x),y-x\rangle
+\frac{\lambda}{2}\norm{y-x}_2^2.
\]
\end{definition}

Our benchmark distinguishes the number of comparator switches from the number of distinct decisions, allowing us to capture recurrence.

\begin{definition}[Comparator paths]
\label{def:comparator-paths}
A \emph{comparator path} is a sequence $u_{1:T}=(u_1,\ldots,u_T)\in \K^T$.
Its number of switches and number of distinct points are, respectively,
\[
k(u_{1:T})
:=
\sum_{t=2}^T \1\{u_t\ne u_{t-1}\},
\qquad
m(u_{1:T})
:=
|\{u_1,\ldots,u_T\}|.
\]
For integers $1\le m\le T$ and $0\le k\le T-1$, define the comparator class
\[
\mathcal U_{m,k}(\K)
:=
\left\{
u_{1:T}\in \K^T:
m(u_{1:T})\le m,\quad
k(u_{1:T})\le k
\right\}.
\]
Thus $\mathcal U_{m,k}(\K)$ consists of paths that switch
at most $k$ times among at most $m$ distinct points.
\end{definition}

The learner is not given the comparator points or their switch times.
We compare its cumulative loss with that of the best admissible
comparator path in hindsight.

\begin{definition}[Tracking regret]
\label{def:regret}
For an algorithm $\mathcal A$ that outputs decisions
$x_{1:T}\in\K^T$ and faces losses $f_{1:T}$,
the tracking regret with at most $k$ switches and $m$ distinct points is
\[
\Reg_T(\mathcal A;m,k)
:=
\sum_{t=1}^T f_t(x_t)
-
\min_{u_{1:T}\in\mathcal U_{m,k}(\K)}
\sum_{t=1}^T f_t(u_t).
\]
\end{definition}

For randomized learners, we consider $\E[\Reg_T(\mathcal A;m,k)]$, where the expectation is over the learner's internal randomness for each fixed loss sequence. We seek learners with low tracking regret, and call a learner computationally efficient if its total running time is polynomial in $T, m, k$ and $d$.

We also define adaptive regret below, which will be useful for our analysis.
Introduced by \citet{hazan2009changing}, it compares the learner with
the best fixed decision on every time interval.

\begin{definition}[Adaptive regret]
\label{def:adaptive-regret}
For an algorithm $\mathcal{A}$ that outputs decisions $x_{1:T}\in\K^T$ and faces losses $f_{1:T}$, the adaptive regret is
\[
    \mathsf{AReg}_T(\mathcal A)
    :=
    \max_{1\le r\le s\le T}
    \left\{
        \sum_{t=r}^s f_t(x_t)
        -\min_{u\in\K}\sum_{t=r}^s f_t(u)
    \right\}.
\]
\end{definition}

\section{Upper Bounds}

We derive standard switching bounds from adaptive-regret guarantees.
Follow-the-Leading-History (FLH)~\citep{hazan2009changing} combines
restarted copies of a base learner using exponential weights and plays
their weighted-average decision. Its pruned version gives the following
guarantee; Appendix~\ref{app:flh} describes the algorithm.

\begin{theorem}[\citealp{hazan2009changing}]
\label{thm:static-to-adaptive}
Suppose the losses are $\alpha$-exp-concave on $\K$.
Given a deterministic base learner $\mathcal B$ with nonnegative,
nondecreasing static-regret bound $R(T)$, pruned FLH produces an
algorithm $\mathcal A$ that maintains $O(\log T)$ active copies
and satisfies
\[
    \mathsf{AReg}_T(\mathcal A)
    \le O\!\left(
        R(T)\log T+\alpha^{-1}\log^2 T
    \right).
\]
The construction preserves oracle efficiency.
\end{theorem}

Adaptive regret bounds regret on each constant segment of a comparator.
Summing over at most $k+1$ segments gives the following switching bound,
proved in Appendix~\ref{app:upper-bound-proofs}.

\begin{corollary}
\label{cor:static-to-switching}
Under the assumptions of Theorem~\ref{thm:static-to-adaptive},
the algorithm $\mathcal A$ satisfies, for every $1\le m\le T$
and $0\le k\le T-1$,
\[
    \Reg_T(\mathcal A;m,k)
    \le (k+1)\mathsf{AReg}_T(\mathcal A)
    \le O\!\left((k+1)\cdot \left[
        R(T)\log T+\alpha^{-1}\log^2 T
    \right]\right).
\]
The bound is independent of $m$, and the algorithm requires
neither $m$ nor $k$ nor the switch times.
\end{corollary}

For strongly convex losses, using projected online gradient descent (OGD)
as the base learner yields a dimension-free switching bound.

\begin{theorem}
\label{prop:strong}
Suppose every $f_t$ is twice differentiable, $\lambda$-strongly convex,
and $G$-Lipschitz on $\K$.
There is a deterministic oracle-efficient algorithm $\mathcal A$ such that,
for every $1\le m\le T$ and $1\le k\le T-1$,
\[
    \Reg_T(\mathcal A;m,k)
    \le O\!\left(\frac{G^2}{\lambda}k\log^2 T\right).
\]
\end{theorem}

\begin{proof}
\citet{hazan2007logarithmic} show that projected online gradient descent, restarted with step sizes
$\eta_t=1/(\lambda t)$ in its local time, has static regret
\[
    R_{\mathrm{OGD}}(n)
    \le O\!\left(\frac{G^2}{\lambda}\log n\right)\,.
\]
Moreover, strong convexity and the gradient bound give
\[
    \nabla^2 f_t(x)\succeq\lambda I
    \succeq\frac{\lambda}{G^2}
        \nabla f_t(x)\nabla f_t(x)^\top,
\]
where for the RHS we used the Cauchy-Schwarz inequality, so each loss is $(\lambda/G^2)$-exp-concave. Applying
Corollary~\ref{cor:static-to-switching} with OGD as the base learner
and $\alpha=\lambda/G^2$ gives the desired algorithm $\mathcal A$
and proves the bound for every $m$ and $k$.
\end{proof}

The bound is independent of the number of distinct comparator points:
strong convexity makes relearning after each switch inexpensive, without
explicitly exploiting recurrence.

For general exp-concave losses, we instead use Online Newton Step (ONS).
Its static-regret bound introduces a factor of $d$, which carries over
to the switching guarantee.

\begin{theorem}
\label{thm:ons-switching-upper}
Suppose $\K$ has Euclidean diameter at most $D$, and every loss $f_t$
is $\alpha$-exp-concave and $G$-Lipschitz on $\K$.
There is a deterministic oracle-efficient algorithm $\mathcal A$ such that,
for every $1\le m\le T$ and $1\le k\le T-1$,
\[
    \Reg_T(\mathcal A;m,k)
    \le O\!\left(d\left(\frac{1}{\alpha}+GD\right)
        k\log^2 T\right).
\]
In particular, this is $O(dk\log^2 T)$ when
$\alpha$, $G$, and $D$ are constants.
\end{theorem}

\begin{proof}
The static-regret guarantee for ONS \citep{hazan2007logarithmic} is
\begin{equation}
\label{eq:ons-static}
    R_{\mathrm{ONS}}(n)
    \le O\!\left(d\left(\frac{1}{\alpha}+GD\right)\log n\right).
\end{equation}
Theorem~\ref{thm:static-to-adaptive}, applied with ONS as the
base learner, therefore gives an algorithm $\mathcal A$ satisfying
\begin{equation}
\label{eq:adaptive-ons}
    \mathsf{AReg}_T(\mathcal A)
    \le O\!\left(d\left(\frac{1}{\alpha}+GD\right)\log^2 T\right),
\end{equation}
where the aggregation term is absorbed using $d\ge1$.
The switching bound follows from Corollary~\ref{cor:static-to-switching}.
\end{proof}

Both algorithms learn each constant segment separately, without
amortizing the learning cost across recurring comparator values.

\section{Computational Hardness for Exp-Concave Functions}
\label{sec:hardness}

As noted in the previous section, combining Follow-the-Leading-History with Online Newton Step gives an efficient algorithm with regret \(\widetilde{\mathcal O}(kd)\) against a comparator that switches at most $k$ times. On the other hand, discretizing the continuous decision set and applying the long-term-memory algorithm of \cite{bousquet2002tracking} gives the regret $\widetilde{\mathcal O}(md+k)$, but requires time exponential in the dimension. In this section, we show that this computational gap is inherent, i.e. even when the comparator takes only $m=2$ distinct values, no polynomial-time algorithm can attain $\widetilde{\mathcal O}(md+k)$ regret.

Our lower bound is based on the hypergraphic planted-clique conjecture of \cite{luo2020hpc}. Fix an integer $r\geq 2$. Let $\mathsf H^{(r)}_0(N)$ denote the distribution over $r$-uniform hypergraphs on $[N]$ in which every $r$-subset of vertices is independently included as a hyperedge with probability $1/2$. Under $\mathsf H^{(r)}_1(N,q)$, we first sample from $\mathsf H^{(r)}_0(N)$, choose a uniformly random set $C\subseteq[N]$ of size $q$, and then force every $r$-subset of $C$ to be a hyperedge. We write $\mathbb E_0,\Pr_0$ and $\mathbb E_1,\Pr_1$ for expectation and probability under the $\mathsf H^{(r)}_0(N)$ and $\mathsf H^{(r)}_1(N,q)$, respectively.

\begin{conjecture}[$r$-HPC~\citep{luo2020hpc}]
\label{conj:hpc}
Fix $r\ge 2$ and $\beta\in(0,1/2)$, and let $q=\lfloor N^\beta\rfloor$. There is no randomized polynomial-time test that, given $H$, outputs ``\textsf{null}" on
\(
H\sim\mathsf H^{(r)}_0(N)
\)
and ``\textsf{planted}" on
\(
H\sim\mathsf H^{(r)}_1(N,q)
\)
with probability tending to one in both cases as $N\to\infty$.
\end{conjecture}

\cite{luo2020hpc} formulate this conjecture for every fixed hyperedge order and clique sizes polynomially below the $\sqrt N$ scale. Our strongest lower bound is obtained by allowing the hyperedge order to vary over fixed constants.

\begin{theorem}[Computational lower bound under $r$-HPC]
\label{thm:r-hpc-lower-bound}
Assume Conjecture~\ref{conj:hpc} for every fixed $r\ge4$.
For every fixed rational $\theta>0$ and fixed $a,b\ge0$ satisfying
\[
\theta a+b<\theta+\frac12,
\]
no randomized polynomial-time OCO algorithm can guarantee
\[
\E[\Reg_T(\mathcal A;2,k)]=\widetilde{\mathcal O}(k^a d^b)
\]
uniformly over instances with $k=\Theta(d^\theta)$.
This holds even for bounded $1$-exp-concave rank-one quadratic
losses on $B_2^d$, with $T=\operatorname{poly}(d)$.
In particular, this rules out any uniform guarantee
$\widetilde O(k^a d^b)$ over all $(k,d)$ whenever
$a<1$ or $b<1/2$, by choosing $\theta$ sufficiently
large or sufficiently small, respectively.
\end{theorem}

The proof of Theorem~\ref{thm:r-hpc-lower-bound} is deferred to Appendix~\ref{app:generalized_hardness}. To expose the main ideas of the reduction with minimal notation, we first prove the case $k=\Theta(d)$, corresponding to $\theta=1$, under the weaker assumption of $4$-HPC alone. This rules out every polynomial-time expected-regret guarantee of the form $\widetilde{\mathcal O}(k^a d^b)$ with $a+b<3/2$ in this regime.

\begin{theorem}[Computational lower bound under $4$-HPC]
\label{thm:4hpc-lower-bound}
Assume Conjecture~\ref{conj:hpc} for $r=4$. For every fixed $a,b\ge0$ with
\(
a+b<\frac32,
\)
no randomized polynomial-time OCO algorithm can guarantee
\[
\E[\Reg_T(\mathcal A;2,k)]=\widetilde{\mathcal O}(k^a d^b)
\]
uniformly over instances with $k=\Theta(d)$.
This holds even for bounded $1$-exp-concave rank-one quadratic
losses on $B_2^d$, with $T=\operatorname{poly}(d)$.
\end{theorem}

We first describe the reduction and establish three lemmas,
then combine them to prove Theorem~\ref{thm:4hpc-lower-bound}.
Throughout, fix $\beta\in(1/3,1/2)$,
$q=\lfloor N^\beta\rfloor$, and $d=q^2$.
\paragraph{Reduction construction.}
We reduce $4$-HPC with planted-clique size $q$ to OCO. The four hyperedge indices play different roles: two determine the comparator's switching pattern, one indexes the OCO coordinates, and the fourth provides repeated observations of each coordinate. Partition $[N]$ deterministically as
\[
[N]=V_1\sqcup V_2\sqcup V_3\sqcup V_4,
\]
with $|V_3|=d$ and the remaining vertices split evenly among $V_1,V_2,V_4$. Since $d=o(N)$, we have $|V_1|,|V_2|,|V_4|=\Theta(N)$. Identify the coordinates of $\mathbb R^d$ with $V_3$. For each $(i,j)\in V_1\times V_2$, create one block containing one round for every $(\ell,v)\in V_3\times V_4$. On this round let
\[
Y_{ij\ell v}
=
\begin{cases}
+1, & \text{if } \{i,j,\ell,v\}\in E(H),\\
-1, & \text{otherwise},
\end{cases}
\qquad
f_{ij\ell v}(x)
=
\frac18(x_\ell-Y_{ij\ell v})^2.
\]
Thus $T=\Theta(N^3d)$, and every queried hyperedge is distinct. Moreover,
\[
\nabla f(x)=\frac{x_\ell-Y}{4}e_\ell,
\qquad
\nabla^2f(x)=\frac14e_\ell e_\ell^\top,
\]
so
\[
\nabla^2f(x)-\nabla f(x)\nabla f(x)^\top
=
\frac{4-(x_\ell-Y)^2}{16}e_\ell e_\ell^\top
\succeq0.
\]
Hence the losses are $1$-exp-concave on $B_2^d$, bounded in $[0,1/2]$, with gradient norm at most $1/2$.

\paragraph{Comparator and gain.}
Under the planted model, suppose $H\sim\mathsf H^{(4)}_1(N,q)$ with planted clique $C$, and write
$C_s=C\cap V_s$ and $q_s=|C_s|$. Define
\[
u^{(0)}=0,
\qquad
u^{(1)}=\gamma\mathbf 1_{C_3},
\qquad
\gamma=\frac{q}{2N}.
\]
Here $\mathbf 1_{C_3}\in\mathbb R^d$ is the indicator vector
with coordinate $1$ for $\ell\in C_3$ and $0$ otherwise, so
$u^{(1)}$ predicts $\gamma$ on clique coordinates and zero elsewhere.
Both comparator points lie in $B_2^d$, since
\[
    \|u^{(1)}\|_2
    =\gamma\sqrt{q_3}
    \le\gamma\sqrt d
    =\frac{q^2}{2N}<1.
\]
Construct a comparator that takes the value $u^{(1)}$ on block $(i,j)$ whenever $(i,j)\in C_1\times C_2$, and takes the value $u^{(0)}$ otherwise.
Figure~\ref{fig:active-block} illustrates one such active block
and its comparator.

\begin{figure}[!htb]
\floatconts
{fig:active-block}
{\caption{An active block $(i,j)\in C_1\times C_2$. Each cell
$(\ell,v)$ is one round with loss
$f_{ij\ell v}(x)=\frac18(x_\ell-Y_{ij\ell v})^2$.
The planted clique forces $+1$ labels on $C_3\times C_4$;
the remaining labels are independent fair signs conditional on
the clique. The fixed comparator vector (right) predicts $\gamma$
on $C_3$ coordinates and zero elsewhere throughout the block.
Hidden clique rows and columns are grouped only for visualization.
See Figure~\ref{fig:planted-block} in
Appendix~\ref{app:construction-illustration} for the full hypergraph-to-loss
construction.}}
{\includegraphics[width=0.9\linewidth]{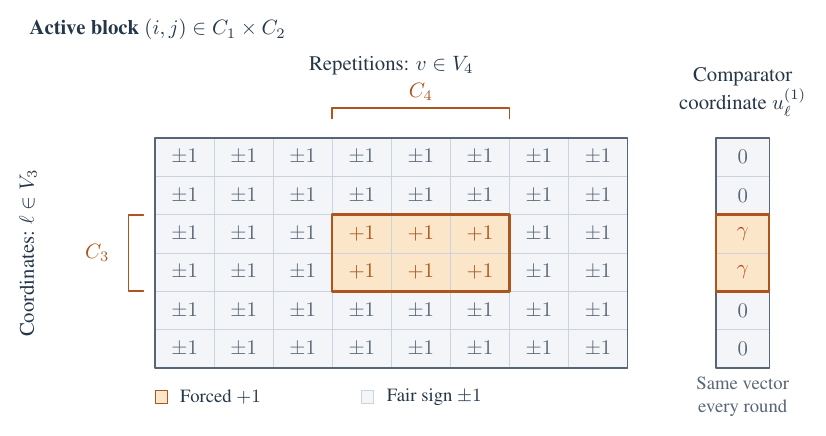}}
\end{figure}

Define the comparator's advantage and the learner's observable gain by
\[
\Delta=\sum_{t=1}^T(f_t(0)-f_t(u_t)),\qquad
G_t=\sum_{s=1}^t(f_s(0)-f_s(x_s)),\qquad G=G_T.
\]
When applied to the learner's gain, $\mathbb E_i$ and
$\Pr_i$ also include the learner's internal randomness. 

We establish three lemmas used in the proof of
Theorem~\ref{thm:4hpc-lower-bound}. We first bound the comparator's number of switches and its advantage over always playing the origin.

\begin{lemma}
\label{lem:4hpc-signal}
Under $H\sim\mathsf H^{(4)}_1(N,q)$, with probability $1-o_N(1)$, the above comparator $u_{1:T}$ makes $O(q^2)$ switches, and satisfies
\[
\Delta
=
\sum_{t=1}^T\bigl(f_t(0)-f_t(u_t)\bigr)
\ge
c\frac{q^7}{N^2} =: S_N
\]
for some constant $c>0$.
\end{lemma}

\begin{proof}[Proof sketch]
Only rounds with $(i,j,\ell)\in C_1\times C_2\times C_3$
contribute to $\Delta$. On every other round, the comparator
predicts zero on the queried coordinate, so
$f_t(u_t)=f_t(0)=1/8$ and the gain is zero.
On a contributing round, the comparator predicts $\gamma$.
Splitting according to whether the fourth vertex lies in the
clique, we obtain
\[
\begin{aligned}
    f_{ij\ell v}(0)-f_{ij\ell v}(u^{(1)})
    &=\frac18\left(Y_{ij\ell v}^2
                  -(\gamma-Y_{ij\ell v})^2\right)=\begin{cases}
        \displaystyle\frac{\gamma}{4}-\frac{\gamma^2}{8},
            &v\in C_4,\\[4pt]
        \displaystyle\frac{\gamma Y_{ij\ell v}}4-\frac{\gamma^2}{8},
            &v\in V_4\setminus C_4.
      \end{cases}
\end{aligned}
\]
Conditional on $C$, the first case is deterministic, while the
labels in the second case are independent unbiased signs.
The quadratic penalty is paid in both cases. Collecting all
deterministic terms therefore gives $\Delta = \mu_C + Z_C$ where
\[
\begin{aligned}
    \mu_C:=\mathbb E[\Delta\mid C]
    &=q_1q_2q_3\left(\frac{\gamma q_4}{4}
                    -\frac{\gamma^2|V_4|}{8}\right), \qquad Z_C =\frac{\gamma}{4}
        \sum_{\substack{i\in C_1,\;j\in C_2,\;\ell\in C_3\\
                        v\in V_4\setminus C_4}}
        Y_{ij\ell v}.
\end{aligned}
\]

With probability $1-o_N(1)$, the clique intersections concentrate
around their means $q|V_s|/N$: in particular,
$q_1,q_2,q_4=\Theta(q)$, $q_3=\Theta(q^3/N)$, and
$q_4\ge q|V_4|/(2N)=\gamma|V_4|$.
Condition on such a clique. Since $\gamma=q/(2N)$ and
$|V_4|=\Theta(N)$, the deterministic contribution satisfies
\[
    \mu_C\ge q_1q_2q_3\frac{\gamma^2|V_4|}{8}
    =\Omega\!\left(q^2\cdot\frac{q^3}{N}\cdot\frac{q^2}{N}\right)
    =\Omega\!\left(\frac{q^7}{N^2}\right).
\]
Thus the positive contribution from forced hyperedges outweighs
the quadratic penalties on all contributing rounds. The remaining term $Z_C$ is a sum of
\[
    q_1q_2q_3(|V_4|-q_4)
    =O\!\left(q^2\cdot\frac{q^3}{N}\cdot N\right)
    =O(q^5)
\]
independent, mean-zero random variables, each of magnitude
$\gamma/4$. Its conditional standard deviation is therefore
$O(\gamma\sqrt{q^5})=O(q^{7/2}/N)=o(\mu_C)$.
A standard Hoeffding bound therefore gives
$\Delta\ge\mu_C/2=\Omega(q^7/N^2)$ with high probability.

Finally, there are $q_1q_2=\Theta(q^2)$ active blocks, throughout
which the comparator stays at $u^{(1)}$. Each active block
contributes at most two switches: one on entry and one on exit.
Consecutive active blocks require no switch between them, so
the total number of switches is $O(q^2)$. Combining these bounds
with the concentration of the clique intersections proves the
claim. See Appendix~\ref{app:4hpc-signal-proof} for the detailed proof.
\end{proof}

The next two lemmas apply to both the present reduction and
its generalization in Appendix~\ref{app:generalized_hardness}.
Under the null model, each label remains an unbiased sign after
the learner chooses its decision. The following lemma shows
that large positive gains over the origin are exponentially
unlikely, with a bound independent of the horizon. Its proof
appears in Appendix~\ref{app:4hpc-null-proof}.

\begin{lemma}
\label{lem:4hpc-null}
Consider losses $f_t(x)=\frac18(x_{\ell_t}-Y_t)^2$ on $B_2^d$, where each coordinate $\ell_t$ is fixed before observing $Y_t$. If, under the null model, $Y_t$ is a fresh unbiased sign conditional on the learner's history and current decision, then every randomized learner satisfies
\[
\mathbb E_0[e^G]\le1,
\qquad
\Pr_0(G\ge z)\le e^{-z}.
\]
\end{lemma}

The identity $G=\Delta-\Reg_T(\mathcal A;u_{1:T})$ shows
that sufficiently small expected regret yields large expected
learner gain on planted instances with substantial comparator
advantage. The next lemma averages independent learner runs on
the same loss sequence to turn this separation into a reliable
polynomial-time test. Its proof appears in
Appendix~\ref{app:4hpc-detection-proof}.

\begin{lemma}
\label{lem:4hpc-detection}
Suppose a polynomial-time reduction produces $T=\operatorname{poly}(N)$
losses in $[0,1/2]$, and let $S_N\to\infty$ be a known threshold.
Assume $\mathbb E_0[e^G]\le1$ under the null model and that,
with probability $1-o_N(1)$ under the planted model, a comparator
has advantage $\Delta\ge S_N$.
A randomized polynomial-time learner with expected regret at most
$S_N/4$ against that comparator on every instance in this event
yields a randomized polynomial-time test distinguishing the models
with probability $1-o_N(1)$.
\end{lemma}

We finally conclude with the proof of Theorem~\ref{thm:4hpc-lower-bound}, which combines the above lemmas
to give the desired lower bound.

\begin{proof}[Proof of Theorem~\ref{thm:4hpc-lower-bound}]
Suppose a randomized polynomial-time learner guarantees
$\widetilde O(k^a d^b)$ expected regret uniformly over instances
with $k=\Theta(d)$, for fixed $a,b\ge0$ with $a+b<3/2$.
Choose $\varepsilon>0$ such that
$a+b<3/2-\varepsilon$, and choose
\[
\beta\in\left(\max\left\{\frac13,\frac1{2+\varepsilon}\right\},\frac12\right).
\]
By Lemma~\ref{lem:4hpc-signal}, on its high-probability event
the comparator uses two points and makes at most $k$ switches,
where $k$ is a sufficiently large constant multiple of $q^2$.
Since $d=q^2$ and $q=\Theta(N^\beta)$, these instances satisfy
$k=\Theta(d)$ and
$T=\Theta(N^3d)=\Theta(d^{1+3/(2\beta)})$.
Thus the assumed guarantee applies and is $\widetilde O(d^{a+b})$.
Moreover, $N=\Theta(d^{1/(2\beta)})$ and $1/\beta<2+\varepsilon$ give
\[
S_N=c\frac{q^7}{N^2}
=\Theta\left(d^{7/2-1/\beta}\right),
\qquad
\frac72-\frac1\beta>\frac32-\varepsilon>a+b.
\]
Consequently, the assumed regret guarantee is $o(S_N)$.
For all sufficiently large $N$, it is at most $S_N/4$.
Lemma~\ref{lem:4hpc-detection} then yields a randomized
polynomial-time test for $4$-HPC, contradicting
Conjecture~\ref{conj:hpc}.
\end{proof}

\section{Conclusion and Future Work}

We studied online convex optimization against comparators that switch at most $k$ times among $m$ distinct decisions. Unrestricted algorithms achieve $\widetilde O(md+k)$ regret, whereas existing efficient methods achieve $\widetilde O(kd)$. Assuming $r$-HPC for every fixed $r\ge4$, we rule out randomized polynomial-time algorithms guaranteeing $\widetilde O(k^a d^b)$ expected regret uniformly for $k=\Theta(d^\theta)$ whenever $\theta>0$ is fixed and rational and $a,b\ge0$ are fixed with $\theta a+b<\theta+1/2$, even for $m=2$.

An immediate question is to determine the optimal polynomial-time regret by closing the gap between our computational lower bound and the $\widetilde O(kd)$ upper bound. A second direction is to characterize the curvature and structural assumptions under which recurrence can be exploited efficiently. Strongly convex losses already admit dimension-free guarantees without exploiting recurrence. Understanding when recurrence yields additional improvements may clarify the computational limits of learning against recurring comparators in online convex optimization.


\section*{AI Disclosure}

GPT-6 Astra was used to explore lower bounds for the tracking-regret problem. In response to a prompt from one of the authors, the model proposed the central idea of a reduction from hypergraphic planted clique. The initial bound was weaker than the one presented here. Through further interaction with the model, the authors refined the argument and strengthened the bound to obtain the result in Theorems~\ref{thm:4hpc-lower-bound} and ~\ref{thm:r-hpc-lower-bound}. The authors verified the correctness of the final result.

OpenAI Codex was also used to help revise the exposition and presentation of proofs. The authors take full responsibility for all content of the paper, including its correctness, originality, attribution, and citations.

\newpage

\bibliographystyle{plainnat}
\begingroup
\setcitestyle{numbers,square}
\bibliography{references}
\endgroup

\appendix
\clearpage
\section{Proofs and Generalization of the Hardness Reduction}
\label{app:hardness-proofs}

We first give the full proofs of Lemmas~\ref{lem:4hpc-signal}--\ref{lem:4hpc-detection},
then extend the construction to prove Theorem~\ref{thm:r-hpc-lower-bound}.

\subsection{Proof of Lemma~\ref{lem:4hpc-signal}}
\label{app:4hpc-signal-proof}

We give an expanded proof of the planted-signal lemma for the
four-uniform construction. The argument has two sources of randomness:
the choice of the planted clique, which determines how many useful
blocks and coordinates there are, and the unforced hyperedges, which
determine the fluctuations of the comparator's gain.

\begin{proof}
Recall that $q=\lfloor N^\beta\rfloor$ with
$\beta\in(1/3,1/2)$, $d=q^2$, and
\[
    |V_3|=d,
    \qquad |V_1|,|V_2|,|V_4|=\Theta(N).
\]
The planted clique $C$ is a uniformly random $q$-subset of $[N]$.
Write $C_s=C\cap V_s$ and $q_s=|C_s|$ for $s\in[4]$.
The comparator uses the two points
\[
    u^{(0)}=0,
    \qquad u^{(1)}=\gamma\mathbf 1_{C_3},
    \qquad \gamma=\frac{q}{2N}.
\]
It plays $u^{(1)}$ throughout block $(i,j)$ if
$(i,j)\in C_1\times C_2$, and plays $u^{(0)}$ throughout every
other block. We call the former blocks active.

\paragraph{Step 1: Concentration of the clique intersections.}
For each $s\in[4]$, the variable $q_s$ is hypergeometric, with mean
\[
    \mu_s:=\mathbb E[q_s]=\frac{q|V_s|}{N}.
\]
Indeed, each vertex of $V_s$ belongs to the uniformly chosen clique
with probability $q/N$. The multiplicative Chernoff bound for
sampling without replacement gives
\[
    \Pr(q_s<\mu_s/2)\le e^{-\mu_s/8},
    \qquad
    \Pr(q_s>2\mu_s)\le e^{-\mu_s/3}.
\]
Define the event
\[
    \mathcal E
    :=\left\{
        \frac{q|V_s|}{2N}\le q_s\le\frac{2q|V_s|}{N}
        \text{ for every }s\in[4]
      \right\}.
\]
A union bound, which does not require the four intersection sizes
to be independent, yields
\[
    \Pr(\mathcal E^c)
    \le 2\sum_{s=1}^4 e^{-\mu_s/8}
    \le 8e^{-\min_s\mu_s/8}
    =o_N(1).
\]
To justify the last equality, the means for $s=1,2,4$ are
$\Theta(q)$, while
\[
    \mu_3=\frac{qd}{N}=\frac{q^3}{N}
    =\Theta(N^{3\beta-1})\longrightarrow\infty.
\]
Thus every mean tends to infinity. This explains the assumption
$\beta>1/3$: it ensures concentration even for the smaller
coordinate part $V_3$.

For the next steps, fix any clique $C$ satisfying $\mathcal E$.
All expectations and probabilities conditional on $C$ are over the
remaining randomness of the hypergraph. On this event,
\[
    q_1,q_2,q_4=\Theta(q),
    \qquad q_3=\Theta\!\left(\frac{q^3}{N}\right).
\]
The constants in these estimates can be chosen uniformly over
all such cliques.

\paragraph{Step 2: Feasibility and number of switches.}
The origin is feasible, and the other point satisfies
\[
    \|u^{(1)}\|_2
    =\gamma\sqrt{q_3}
    \le\gamma\sqrt d
    =\frac{q^2}{2N}
    \longrightarrow0,
\]
where the limit uses $\beta<1/2$. Hence both points lie in
$B_2^d$ for sufficiently large $N$.

There are exactly $q_1q_2=\Theta(q^2)$ active blocks. The
comparator is constant within each block, so a switch can occur
only at a block boundary. Each active block can account for at
most two switches, one on entering it and one on leaving it.
Consecutive active blocks require no switch between them.
Consequently, the total number of switches is at most
$2q_1q_2=O(q^2)$. In fact, this upper bound holds for every
clique, since $q_1,q_2\le q$.

\paragraph{Step 3: Gain for a fixed active block and coordinate.}
Outside active blocks, the comparator equals the origin and its
gain is zero. On an active block, a round with $\ell\notin C_3$
also has zero gain, because $u^{(1)}_\ell=0$ and the loss depends
only on coordinate $\ell$. Thus only triples
$(i,j,\ell)\in C_1\times C_2\times C_3$ contribute to
\[
    \Delta=\sum_{t=1}^T\bigl(f_t(0)-f_t(u_t)\bigr).
\]
Fix one such triple. For every $v\in V_4$, expanding the square
gives
\[
\begin{aligned}
    f_{ij\ell v}(0)-f_{ij\ell v}(u^{(1)})
    &=\frac18\left(Y_{ij\ell v}^2
                  -(\gamma-Y_{ij\ell v})^2\right)\\
    &=\frac{\gamma Y_{ij\ell v}}4-\frac{\gamma^2}{8}.
\end{aligned}
\]
If $v\in C_4$, all four vertices lie in $C$, so the hyperedge is
forced to be present and $Y_{ij\ell v}=1$. If $v\notin C_4$,
the hyperedge is not forced; conditional on $C$, its label is
an unbiased sign. Therefore
\[
    \mathbb E[Y_{ij\ell v}\mid C]
    =\begin{cases}
        1,&v\in C_4,\\
        0,&v\notin C_4.
      \end{cases}
\]
Summing over the $|V_4|$ rounds for this triple, we obtain
\[
\begin{aligned}
    \mathbb E\!\left[
        \sum_{v\in V_4}
        \bigl(f_{ij\ell v}(0)-f_{ij\ell v}(u^{(1)})\bigr)
        \,\middle|\,C\right]
    &=\frac{\gamma q_4}{4}-\frac{\gamma^2|V_4|}{8}\\
    &\ge\frac{\gamma^2|V_4|}{8}\\
    &=\frac{q^2|V_4|}{32N^2}
      =\Omega\!\left(\frac{q^2}{N}\right).
\end{aligned}
\]
The inequality follows from
$q_4\ge q|V_4|/(2N)=\gamma|V_4|$, which holds on
$\mathcal E$. The positive bias from the forced hyperedges
therefore outweighs the quadratic penalty paid on all rounds.
This is a statement about the sum in expectation; an individual
round with label $-1$ has negative gain.

\paragraph{Step 4: Total expected gain.}
There are
\[
    M_C:=q_1q_2q_3
    =\Theta(q^2)\,\Theta\!\left(\frac{q^3}{N}\right)
    =\Theta\!\left(\frac{q^5}{N}\right)
\]
contributing triples. By linearity of expectation,
\[
\begin{aligned}
    \mu_C:=\mathbb E[\Delta\mid C]
    &=M_C\left(\frac{\gamma q_4}{4}
              -\frac{\gamma^2|V_4|}{8}\right)\\
    &\ge M_C\frac{q^2|V_4|}{32N^2}
     \ge c_0\frac{q^7}{N^2}
\end{aligned}
\]
for a constant $c_0>0$ independent of $N$ and of the clique
satisfying $\mathcal E$. In particular, the power $q^7/N^2$
comes from multiplying the number of active blocks, the number
of clique coordinates, and the expected gain for each such pair:
\[
    \underbrace{\Theta(q^2)}_{\text{active blocks}}
    \;\underbrace{\Theta(q^3/N)}_{\text{clique coordinates}}
    \;\underbrace{\Omega(q^2/N)}_{\text{gain per pair}}
    =\Omega\!\left(\frac{q^7}{N^2}\right).
\]

\paragraph{Step 5: Concentration of the total gain.}
It remains to show that the realized gain is unlikely to fall
far below its conditional mean. Define
\[
    \mathcal I_C
    :=C_1\times C_2\times C_3\times(V_4\setminus C_4),
    \qquad n_C:=|\mathcal I_C|.
\]
Conditional on $C$, both the gain from forced edges and all
quadratic penalty terms are deterministic. Hence
\[
    \Delta-\mu_C
    =\sum_{(i,j,\ell,v)\in\mathcal I_C}
      \frac{\gamma Y_{ij\ell v}}4.
\]
Every summand has mean zero and takes values in
$\{-\gamma/4,\gamma/4\}$. These summands are independent
conditional on $C$: they correspond to unforced hyperedges,
and distinct tuples give distinct hyperedges because the
four partition parts are disjoint. Their number satisfies
\[
    n_C=M_C(|V_4|-q_4)
    \le M_C|V_4|=O(q^5).
\]

When $n_C>0$, Hoeffding's inequality gives, for every $a>0$,
\[
    \Pr(\Delta-\mu_C\le-a\mid C)
    \le\exp\!\left(-\frac{2a^2}{n_C(\gamma/2)^2}\right)
    =\exp\!\left(-\frac{8a^2}{n_C\gamma^2}\right).
\]
The denominator uses the squared range length of each summand,
which is $(\gamma/2)^2$. If $n_C=0$, then
$\Delta=\mu_C$ deterministically and no concentration bound
is needed.

Choose $c=c_0/2$ and set $S_N=cq^7/N^2$. Since
$\mu_C\ge c_0q^7/N^2$, we have $S_N\le\mu_C/2$.
Using $a=\mu_C/2$ in the preceding bound yields
\[
\begin{aligned}
    \Pr(\Delta<S_N\mid C)
    &\le\Pr(\Delta-\mu_C\le-\mu_C/2\mid C)\\
    &\le\exp\!\left(-\frac{2\mu_C^2}{n_C\gamma^2}\right)\\
    &\le\exp\!\left(-\Omega\!\left(
        \frac{(q^7/N^2)^2}{q^5(q^2/N^2)}\right)\right)\\
    &=\exp\!\left(-\Omega\!\left(\frac{q^7}{N^2}\right)\right)
      =o_N(1).
\end{aligned}
\]
Here $\gamma^2=q^2/(4N^2)$, and
$q^7/N^2=\Theta(N^{7\beta-2})\to\infty$ because
$\beta>1/3$. The bound is uniform over cliques satisfying
$\mathcal E$. Averaging over those cliques and adding the
probability of $\mathcal E^c$ therefore gives
\[
    \Pr(\Delta<S_N)
    \le\Pr(\mathcal E^c)
      +\exp\!\left(-\Omega\!\left(\frac{q^7}{N^2}\right)\right)
    =o_N(1).
\]
Together with feasibility and the switch bound, this proves
Lemma~\ref{lem:4hpc-signal}.
\end{proof}

\subsection{Proof of Lemma~\ref{lem:4hpc-null}}
\label{app:4hpc-null-proof}

We expand the exponential-moment argument, including the
conditioning needed for a randomized learner. The key point
is that each new label remains an unbiased sign after the
learner has chosen its current decision, even though that
decision may depend on all earlier labels.

\begin{proof}
Consider the losses
\[
    f_t(x)=\frac18(x_{\ell_t}-Y_t)^2,
\]
and let $\mathcal H_t$ be the information available just before
$Y_t$ is revealed, including the learner's history, its current
decision $x_t$, and the coordinate $\ell_t$. Thus
$b_t:=x_{t,\ell_t}$ is $\mathcal H_t$-measurable. For a
randomized learner, conditioning on $x_t$ includes the outcome
of its random choice for this round; we will average over that
choice when taking unconditional expectations.

Under the hypothesis of the lemma,
\[
    \Pr_0(Y_t=1\mid\mathcal H_t)
    =\Pr_0(Y_t=-1\mid\mathcal H_t)=\frac12.
\]
In our reduction, this property follows from the independence
of hyperedges under the null model and the fact that each round
queries a previously unqueried hyperedge. In particular,
observing past losses does not reveal the next label.

\paragraph{Step 1: Express the gain on one round.}
Since $Y_t^2=1$, the origin has loss $f_t(0)=1/8$. Expanding
the square gives
\[
\begin{aligned}
    g_t:=f_t(0)-f_t(x_t)
    &=\frac18\left(Y_t^2-(b_t-Y_t)^2\right)\\
    &=\frac{b_tY_t}{4}-\frac{b_t^2}{8}.
\end{aligned}
\]
Consequently,
\[
    \mathbb E_0[g_t\mid\mathcal H_t]
    =-\frac{b_t^2}{8}\le0.
\]
The linear term is a mean-zero random reward, while the
quadratic term is a deterministic penalty once the prediction
has been chosen. Nonpositive expected gain alone does not give
the desired exponential upper-tail bound, so we next control
the exponential moment of $g_t$.

\paragraph{Step 2: Bound the conditional exponential moment.}
Averaging over the two equally likely values of $Y_t$ yields
\[
\begin{aligned}
    \mathbb E_0[e^{g_t}\mid\mathcal H_t]
    &=\frac12e^{b_t/4-b_t^2/8}
      +\frac12e^{-b_t/4-b_t^2/8}\\
    &=e^{-b_t^2/8}\cosh(b_t/4),
\end{aligned}
\]
where $\cosh(z)=(e^z+e^{-z})/2$. For every real $z$,
\[
    \cosh(z)
    =\sum_{j=0}^{\infty}\frac{z^{2j}}{(2j)!}
    \le\sum_{j=0}^{\infty}\frac{(z^2/2)^j}{j!}
    =e^{z^2/2}.
\]
The termwise inequality uses $(2j)!\ge2^j j!$ and
$z^{2j}\ge0$. Substituting $z=b_t/4$ gives
\[
    \mathbb E_0[e^{g_t}\mid\mathcal H_t]
    \le e^{-b_t^2/8+b_t^2/32}
    =e^{-3b_t^2/32}
    \le1.
\]
This bound holds for every realized history and every current
decision, so it allows arbitrary adaptive and randomized
choices by the learner.

\paragraph{Step 3: Iterate the conditional bound.}
Define
\[
    G_0=0,
    \qquad G_t=\sum_{s=1}^t g_s,
    \qquad G=G_T.
\]
The accumulated gain $G_{t-1}$ is known given $\mathcal H_t$.
Since $G_t=G_{t-1}+g_t$, the tower property gives
\[
\begin{aligned}
    \mathbb E_0[e^{G_t}]
    &=\mathbb E_0\!\left[
        \mathbb E_0[e^{G_{t-1}}e^{g_t}\mid\mathcal H_t]
      \right]\\
    &=\mathbb E_0\!\left[
        e^{G_{t-1}}\mathbb E_0[e^{g_t}\mid\mathcal H_t]
      \right]\\
    &\le\mathbb E_0[e^{G_{t-1}}].
\end{aligned}
\]
Iterating this inequality from $t=T$ down to $t=1$ yields
\[
    \mathbb E_0[e^G]
    \le\mathbb E_0[e^{G_{T-1}}]
    \le\cdots\le\mathbb E_0[e^{G_0}]=1.
\]
The gains $g_t$ need not be independent, since the learner
can adapt to previous outcomes. Conditioning at each round
is what makes the iteration valid despite this dependence.
All expectations are finite because the losses lie in
$[0,1/2]$ and the horizon is finite.

\paragraph{Step 4: Derive the upper-tail bound.}
The exponential is increasing and $e^G$ is nonnegative.
Markov's inequality therefore gives, for every real $z$,
\[
    \Pr_0(G\ge z)
    =\Pr_0(e^G\ge e^z)
    \le\frac{\mathbb E_0[e^G]}{e^z}
    \le e^{-z}.
\]
In particular, for any $\delta\in(0,1)$, the probability that
the learner gains at least $\log(1/\delta)$ over the origin
is at most $\delta$, independently of the horizon $T$.
This proves both conclusions of Lemma~\ref{lem:4hpc-null}.
\end{proof}

\subsection{Proof of Lemma~\ref{lem:4hpc-detection}}
\label{app:4hpc-detection-proof}

\begin{proof}
Let $\mathcal E_N$ be the planted event in the statement.
For every fixed instance in $\mathcal E_N$, the identity
$G=\Delta-\Reg_T(\mathcal A;u_{1:T})$ gives
$\mathbb E[G]\ge3S_N/4$, where the expectation is over the
learner's randomness. Run
$J=\lceil64T^2S_N^{-2}\log N\rceil$ independent copies on
the same loss sequence, and set
$\overline G=J^{-1}\sum_{j=1}^J G^{(j)}$.
Since each gain lies in $[-T/2,T/2]$, Hoeffding's inequality,
applied to each fixed instance in $\mathcal E_N$, gives
\[
\Pr_1\!\left(\overline G<S_N/2\mid\mathcal E_N\right)
\le \exp\left(-\frac{JS_N^2}{8T^2}\right)
\le N^{-8}.
\]
Under the null model, convexity of the exponential gives
\[
\mathbb E_0[e^{\overline G}]
\le\frac1J\sum_{j=1}^J\mathbb E_0[e^{G^{(j)}}]\le1,
\qquad
\Pr_0(\overline G\ge S_N/2)\le e^{-S_N/2}=o_N(1),
\]
where the probability bound follows from Markov's inequality.
Thus the test that outputs planted exactly when
$\overline G\ge S_N/2$ succeeds with probability $1-o_N(1)$
in both models. Since $S_N\to\infty$ and $T$ is polynomial
in $N$, $J$ is polynomially bounded, so the test runs in
randomized polynomial time.
\end{proof}

\subsection{Generalized construction and proof of Theorem~\ref{thm:r-hpc-lower-bound}}
\label{app:generalized_hardness}

We extend the construction for Theorem~\ref{thm:4hpc-lower-bound}
to prove Theorem~\ref{thm:r-hpc-lower-bound}. As in the main text,
we first describe the reduction and establish its planted-signal
lemma, then combine it with the two general lemmas from the main
text to prove the theorem.
Throughout, fix positive integers $p,s$, and set
\[
h=\left\lceil\frac{s}{2}\right\rceil+1,\qquad r=p+s+h.
\]
Fix $\beta\in(1/3,1/2)$ such that
$\beta(p+3s+2h)>s+h$. Such a choice is possible because $p+3s+2h > 2(s+h)$. Let
$q=\lfloor N^\beta\rfloor$ and $d=q^{2s}$.
The additional condition ensures that the planted advantage
tends to infinity.

\paragraph{Reduction construction.}
We reduce $r$-HPC with planted-clique size $q$ to OCO. The $r$ hyperedge indices are divided into three groups: $p$ indices determine the comparator's switching pattern, $s$ indices jointly index the OCO coordinates, and $h$ indices provide repeated observations of each coordinate. Partition $[N]$ deterministically as
\[
[N]
=
B_1\sqcup\cdots\sqcup B_p
\sqcup
D_1\sqcup\cdots\sqcup D_s
\sqcup
W_1\sqcup\cdots\sqcup W_h,
\]
where \(|D_j|=q^2\) for every \(j\in[s],\) and split the remaining vertices evenly among the $B_a$ and $W_c$. Since $q^2=o(N)$ and $p,s,h$ are fixed, \(|B_a|,|W_c|=\Theta(N).\)

Identify the coordinates of $\mathbb R^d$ with \(\prod_{i=1}^s D_i\) whose cardinality is $q^{2s}=d$. For each
\(
\mathbf i=(i_1,\ldots,i_p)
\in \prod_{i=1}^p B_i
\)
create one block containing one round for every
\(
\boldsymbol\ell=(\ell_1,\ldots,\ell_s)
\in \prod_{i=1}^s D_i
\)
and
\(
\mathbf v=(v_1,\ldots,v_h)
\in \prod_{i=1}^h W_i.
\)
On this round let
\[
Y_{\mathbf i,\boldsymbol\ell,\mathbf v}
=
\begin{cases}
+1,
&
\text{if }
\{i_1,\ldots,i_p,\ell_1,\ldots,\ell_s,v_1,\ldots,v_h\}
\in E(H),
\\
-1,
&
\text{otherwise},
\end{cases}\,, \qquad 
f_{\mathbf i,\boldsymbol\ell,\mathbf v}(x)
=
\frac18
\left(
x_{\boldsymbol\ell}
-
Y_{\mathbf i,\boldsymbol\ell,\mathbf v}
\right)^2.
\]
Since the partition sets are disjoint, each hyperedge queried by the construction uniquely determines the tuple $(\mathbf i,\boldsymbol\ell,\mathbf v)$, so every queried hyperedge is distinct. Thus \(T=\Theta(N^{p+h}d).\) Moreover,
\[
\nabla f(x)
=
\frac{x_{\boldsymbol\ell}-Y}{4}e_{\boldsymbol\ell},
\qquad
\nabla^2 f(x)
=
\frac14 e_{\boldsymbol\ell}e_{\boldsymbol\ell}^\top,
\]
so
\[
\nabla^2f(x)-\nabla f(x)\nabla f(x)^\top
=
\frac{4-(x_{\boldsymbol\ell}-Y)^2}{16}
e_{\boldsymbol\ell}e_{\boldsymbol\ell}^\top
\succeq0.
\]
Hence the losses are $1$-exp-concave on $B_2^d$, bounded in $[0,1/2]$, with gradient norm at most $1/2$.

\paragraph{Comparator and gain.}
Now suppose
\(
H\sim\mathsf H^{(r)}_1(N,q)
\)
with planted clique $C$. Define
\[
C_{B_a}:=C\cap B_a,
\qquad
C_{D_j}:=C\cap D_j,
\qquad
C_{W_c}:=C\cap W_c,
\]
and let
\(
\mathcal C_D
:=
\prod_{i=1}^s C_{D_i}.
\)
Set
\[
u^{(0)}=0,
\qquad
u^{(1)}
=
\gamma\mathbf 1_{\mathcal C_D},
\qquad
\gamma
=
\left(\frac {q}{2N}\right)^h\,.
\]
Construct a comparator that takes the value $u^{(1)}$ on block $\mathbf i$ whenever
\(
\mathbf i
\in
\prod_{i=1}^p C_{B_i},
\)
and takes the value $u^{(0)}$ otherwise. 
Define the comparator's advantage and the learner's observable gain by
\[
\Delta=\sum_{t=1}^T(f_t(0)-f_t(u_t)),\qquad
G_t=\sum_{\tau=1}^t(f_\tau(0)-f_\tau(x_\tau)),\qquad G=G_T.
\]
The following lemmas concern this construction.
We first verify the comparator's feasibility, switch count,
and planted advantage.

\begin{lemma}
\label{lem:r-hpc-signal}
Under
\(
H\sim\mathsf H^{(r)}_1(N,q),
\)
with probability $1-o_N(1)$, the comparator $u_{1:T}$ takes values in $B_2^d$, makes $O(q^p)$ switches, and satisfies
\[
\Delta
=
\sum_{t=1}^T
\bigl(f_t(0)-f_t(u_t)\bigr)
\ge
c
\frac{q^{p+3s+2h}}{N^{s+h}}
=:S_N
\]
for some constant $c>0$.
\end{lemma}

\begin{proof}
Since $C$ is a uniformly random $q$-subset of $[N]$,
$\mathbb E[|C\cap V|]=q|V|/N$ for each partition part $V$.
Thus, the expected sizes of $C_{B_a}$ and $C_{W_c}$ are $\Theta(q)$,
and those of $C_{D_j}$ are $\Theta(q^3/N)\to\infty$ because $\beta>1/3$.
A Chernoff bound therefore gives,
with probability $1-o_N(1)$,
\[
\frac{q|V|}{2N}\le |C\cap V|\le\frac{2q|V|}{N}
\qquad\text{for every partition part }V.
\]
Condition on a clique $C$ satisfying this event. In particular,
$|C_{B_a}|,|C_{W_c}|=\Theta(q)$ and
$|\mathcal C_D|=\Theta(q^{3s}/N^s)$. Moreover,
\[
\|u^{(1)}\|_2^2
=\gamma^2|\mathcal C_D|
\le 2^{s-2h}\frac{q^{2h+3s}}{N^{2h+s}}
\le N^{s/2-h}=o_N(1),
\]
where $\beta<1/2$ gives $q\le N^{1/2}$, and $h>s/2$.
Thus both comparator points lie in $B_2^d$ for sufficiently
large $N$. There are $\prod_{a=1}^p|C_{B_a}|=\Theta(q^p)$
active blocks, so the comparator makes $O(q^p)$ switches.

Outside active blocks the comparator is the origin, and outside
$\mathcal C_D$ its coordinates are zero. Thus only active
block--planted-coordinate pairs contribute to $\Delta$.
Write $\mathcal W=\prod_{c=1}^h W_c$,
$M_W=\prod_{c=1}^h|C_{W_c}|$, and $L_W=|\mathcal W|$.
For each such pair $(\mathbf i,\boldsymbol\ell)$,
\[
\begin{aligned}
\mathbb E\!\left[\sum_{\mathbf v\in\mathcal W}
\bigl(f_{\mathbf i,\boldsymbol\ell,\mathbf v}(0)
-f_{\mathbf i,\boldsymbol\ell,\mathbf v}(u^{(1)})\bigr)
\,\middle|\,C\right]
&=\mathbb E\!\left[\sum_{\mathbf v\in\mathcal W}
\left(\frac{\gamma Y_{\mathbf i,\boldsymbol\ell,\mathbf v}}4
-\frac{\gamma^2}8\right)\,\middle|\,C\right] \\
&=\frac{\gamma M_W}{4}-\frac{\gamma^2L_W}{8}
\ge\frac{\gamma^2L_W}{8}
=\Omega\left(\frac{q^{2h}}{N^h}\right),
\end{aligned}
\]
where the concentration event gives
$M_W\ge(q/2N)^hL_W=\gamma L_W$.
Since there are
$(\prod_{a=1}^p|C_{B_a}|)|\mathcal C_D|
=\Theta(q^{p+3s}/N^s)$ such pairs, some constant $c_0>0$ satisfies
\[
\mu_C:=\mathbb E[\Delta\mid C]
\ge c_0\frac{q^{p+3s+2h}}{N^{s+h}}.
\]

Conditional on $C$, the random part of $\Delta$ is a sum of at most
$O(q^{p+3s}N^{h-s})$ independent terms, each taking values in
$\{-\gamma/4,\gamma/4\}$.
Choose $c=c_0/2$. Hoeffding's inequality then gives
\[
\begin{aligned}
\Pr\!\left(\Delta<c\frac{q^{p+3s+2h}}{N^{s+h}}
\,\middle|\,C\right)
&\le\exp\left(-\Omega\left(
\frac{(q^{p+3s+2h}/N^{s+h})^2}
{q^{p+3s}N^{h-s}\gamma^2}\right)\right) \\
&=\exp\left(-\Omega\left(
\frac{q^{p+3s+2h}}{N^{s+h}}\right)\right)=o_N(1).
\end{aligned}
\]
The final equality follows from
$\beta>(s+h)/(p+3s+2h)$, which ensures
$q^{p+3s+2h}/N^{s+h}\to\infty$.
Combining this with the concentration event proves the claim.
\end{proof}

The remaining steps use Lemmas~\ref{lem:4hpc-null}
and~\ref{lem:4hpc-detection} directly. The losses have the same
coordinate-quadratic form, and distinct queried hyperedges give
fresh unbiased signs under the null model. The planted event
is supplied by Lemma~\ref{lem:r-hpc-signal}; moreover,
$S_N\to\infty$ by the condition on $\beta$, and the reduction
has polynomial length and running time for fixed $p,s,h$.

\begin{proof}[Proof of Theorem~\ref{thm:r-hpc-lower-bound}]
Suppose a randomized polynomial-time learner guarantees
$\widetilde O(k^a d^b)$ expected regret uniformly over instances
with $k=\Theta(d^\theta)$, for a fixed rational $\theta>0$
and fixed $a,b\ge0$ satisfying $\theta a+b<\theta+1/2$.
Choose positive integers $p,s$ with $p/(2s)=\theta$.
Then the hypothesis gives
\[
pa+2sb=2s(\theta a+b)<2s(\theta+1/2)=p+s.
\]
Set $h=\lceil s/2\rceil+1$ and $r=p+s+h\ge4$, and choose
\[
\beta\in\left(
\max\left\{\frac13,
\frac{s+h}{p+3s+2h-pa-2sb}\right\},\frac12\right).
\]
This interval is nonempty because
$p+3s+2h-pa-2sb>2(s+h)$.
It also satisfies the condition on $\beta$ used in the construction.

By Lemma~\ref{lem:r-hpc-signal}, on its high-probability event
the comparator uses two points and makes at most $k$ switches,
where $k$ is a sufficiently large constant multiple of $q^p$.
Since $d=q^{2s}$, we have $k=\Theta(d^{p/(2s)})=\Theta(d^\theta)$.
Moreover, $q=\Theta(N^\beta)$ gives
$T=\Theta(N^{p+h}d)=\Theta(d^{1+(p+h)/(2s\beta)})$,
which is polynomial in $d$.
Thus the assumed guarantee applies and is
$\widetilde O(q^{pa+2sb})$, whereas
\[
S_N=c\frac{q^{p+3s+2h}}{N^{s+h}},\qquad
\beta(p+3s+2h-pa-2sb)>s+h.
\]
Consequently, the assumed regret guarantee is $o(S_N)$.
For all sufficiently large $N$, it is at most $S_N/4$.
Lemma~\ref{lem:4hpc-detection} then yields a randomized
polynomial-time test for $r$-HPC, contradicting
Conjecture~\ref{conj:hpc}.
\end{proof}

\section{Follow-the-Leading-History and the Switching Bound}
\label{app:flh}

We recall the pruned Follow-the-Leading-History (FLH)
algorithm of \citet{hazan2009changing}, used in
Theorem~\ref{thm:static-to-adaptive}.
Each expert $E_i$ is a fresh copy of the base learner
$\mathcal B$, initialized at round $i$ and trained only
on subsequent losses.

\paragraph{Pruning schedule.}
Write $i=2^{v(i)}q$ with $q$ odd, and set
$L_i=2^{v(i)+2}+1$.
The active set on round $t$ is
\[
    S_t=\{i\in[t]:t\le i+L_i\}.
\]
Thus experts are discarded after their lifetimes expire.
This schedule maintains $O(\log T)$ active experts and
satisfies
\[
    S_t\cap
    \left[r,\left\lfloor\frac{r+t}{2}\right\rfloor\right]
    \ne\varnothing
    \qquad\text{for every }1\le r\le t.
\]

\begin{algorithm}[htbp]
\floatconts
{alg:flh}%
{\caption{Pruned Follow-the-Leading-History (FLH)}}%
{%
\begin{algorithmic}[1]
\STATE \textbf{Input:} Base Learner $\mathcal B$,
exp-concavity parameter $\alpha>0$, horizon $T$
\STATE Initialize $E_1\gets\mathcal B$, $S_1\gets\{1\}$,
and $p_1^{(1)}\gets1$
\FOR{$t=1,\ldots,T$}
    \STATE Obtain $x_t^{(i)}$ from each $E_i$, $i\in S_t$
    \STATE Play
    $x_t\gets\sum_{i\in S_t}p_t^{(i)}x_t^{(i)}$
    \STATE Observe $f_t$ and update every active $E_i$
    \IF{$t<T$}
        \STATE \textbf{Multiplicative update:} For all $i\in S_t$,
        \[
            \widehat p_{t+1}^{(i)}
            \gets
            \frac{p_t^{(i)}
            e^{-\alpha f_t(x_t^{(i)})}}
            {\sum_{j\in S_t}p_t^{(j)}
            e^{-\alpha f_t(x_t^{(j)})}}.
        \]
        \STATE \textbf{Addition:} Initialize
        $E_{t+1}\gets\mathcal B$ and set
        \[
            \widetilde p_{t+1}^{(t+1)}
            \gets\frac{1}{t+1},
            \qquad
            \widetilde p_{t+1}^{(i)}
            \gets\frac{t}{t+1}\widehat p_{t+1}^{(i)}
            \quad(i\in S_t).
        \]
        \STATE \textbf{Pruning:} Set
        $S_{t+1}\gets
        \{i\in S_t\cup\{t+1\}:t+1\le i+L_i\}$
        \STATE Discard experts outside $S_{t+1}$ and set
        \[
            p_{t+1}^{(i)}
            \gets
            \frac{\widetilde p_{t+1}^{(i)}}
            {\sum_{j\in S_{t+1}}\widetilde p_{t+1}^{(j)}}
            \quad(i\in S_{t+1}).
        \]
    \ENDIF
\ENDFOR
\end{algorithmic}
}%
\end{algorithm}

Exp-concavity controls the aggregation cost, while the
pruning schedule provides interval coverage with logarithmic
computational overhead. The resulting adaptive-regret bound
is stated in Theorem~\ref{thm:static-to-adaptive}.

\FloatBarrier
\subsection{Proof of the Switching Corollary}
\label{app:upper-bound-proofs}

\begin{proof}[Proof of Corollary~\ref{cor:static-to-switching}]
Fix any comparator path $u_{1:T}\in\mathcal U_{m,k}(\K)$.
Let $I_1,\ldots,I_s$ be its maximal constant segments, where
$s\le k+1$, and let $v_j$ be its constant value on $I_j$.
Since $v_j$ is one possible fixed decision on $I_j$, the
definition of adaptive regret gives
\[
\begin{aligned}
    \sum_{t\in I_j}\bigl(f_t(x_t)-f_t(v_j)\bigr)
    &\le \sum_{t\in I_j}f_t(x_t)
        -\min_{v\in\K}\sum_{t\in I_j}f_t(v) \le \mathsf{AReg}_T(\mathcal A).
\end{aligned}
\]
Summing over the segments and considering the best comparator path in the set $\mathcal U_{m,k}(\K)$ gives the claimed bound.
\end{proof}

\clearpage
\section{From a Hypergraph to an OCO Block}
\label{app:construction-illustration}

Figure~\ref{fig:planted-block} illustrates the four-uniform
construction, tracing a tuple $(i,j,\ell,v)$ from the hypergraph
to its quadratic loss and its position in an active OCO block.

\begin{figure}[!htbp]
\floatconts
{fig:planted-block}
{\caption{From a hypergraph to a loss sequence. Blue outlines link
one queried four-vertex hyperedge to its cell in an OCO block.
Orange vertices belong to the hidden clique, forcing $+1$ labels
on $C_3\times C_4$; all remaining cells contain independent fair
signs conditional on the clique. The comparator predicts $\gamma$
on $C_3$ rows throughout the block. Only one hyperedge is drawn.
Clique membership is shown for explanation, and the schematic
grid groups clique rows and columns for visualization.}}
{\includegraphics[width=0.92\linewidth]{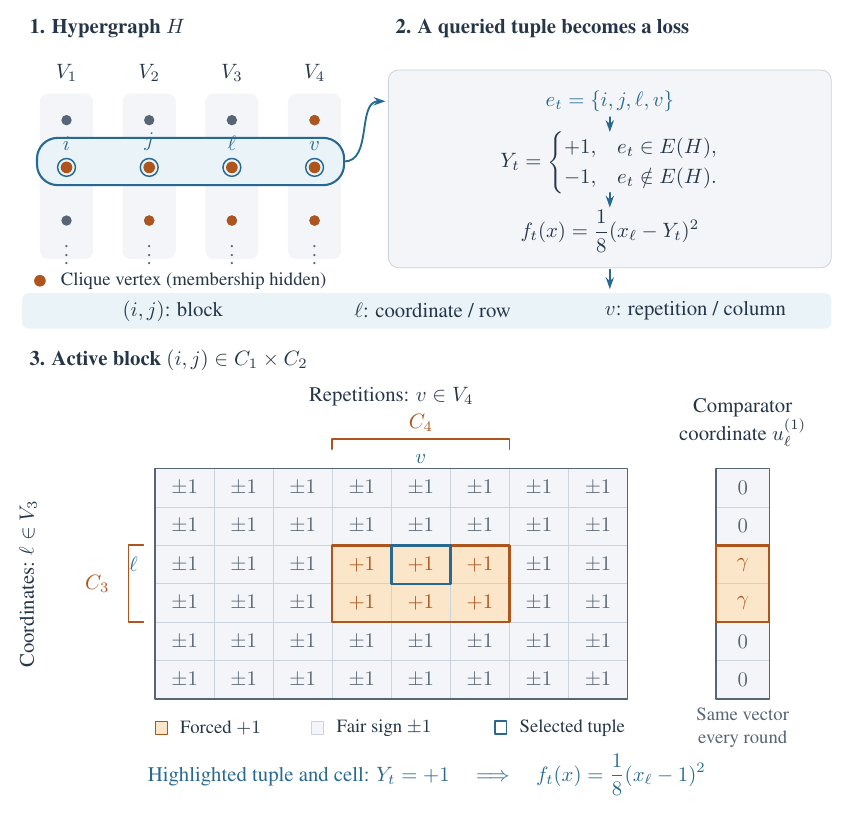}}
\end{figure}

The positive bias on $C_3$ rows makes predicting $\gamma$
profitable in expectation. On an inactive block, where
$(i,j)\notin C_1\times C_2$, all labels are unbiased and the
comparator plays the origin.

\end{document}